\documentclass[twoside]{article}

\usepackage[accepted]{aistats2022}
\usepackage[round]{natbib}

\usepackage{amsfonts}
\usepackage{amsmath}
\usepackage{amsthm}
\usepackage{amssymb}
\usepackage{algorithm}
\usepackage[noend]{algpseudocode}
\usepackage{xcolor}
\usepackage{booktabs}
\usepackage{graphicx}
\usepackage[capitalise,nameinlink]{cleveref}
\usepackage{subcaption}

\crefname{equation}{Equation}{Equations}
\crefname{figure}{Figure}{Figures}

\algnewcommand\algorithmicinput{\textbf{Input:}}
\algnewcommand\algorithmicoutput{\textbf{Output:}}
\algnewcommand\Input{\item[\algorithmicinput]}%
\algnewcommand\Output{\item[\algorithmicoutput]}%

\theoremstyle{definition}

\newtheorem{prop}{Proposition}
\newtheorem{defn}{Definition}
\newtheorem{lem}{Lemma}
\newcommand{\reprop}[3]{\newtheorem*{#1}{Proposition \ref{#1}}\begin{#1}[#2]#3\end{#1}}

\DeclareMathOperator*{\argmax}{arg\,max}
\DeclareMathOperator*{\argmin}{arg\,min}

\newcommand{\PC}{\ensuremath{\mathcal{C}}}
\newcommand{\Mcircuit}{\ensuremath{\mathcal{M}}}

\newcommand{\EB}{\mathsf{EB}}

\newcommand{\X}{\ensuremath{\mathbf{X}}}
\newcommand{\Y}{\ensuremath{\mathbf{Y}}}

\newcommand{\E}{\ensuremath{\mathbf{E}}}
\newcommand{\Q}{\ensuremath{\mathbf{Q}}}
\newcommand{\Hs}{\ensuremath{\mathbf{H}}}

\newcommand{\x}{\ensuremath{\boldsymbol{x}}}
\newcommand{\y}{\ensuremath{\boldsymbol{y}}}

\newcommand{\e}{\ensuremath{\boldsymbol{e}}}
\newcommand{\q}{\ensuremath{\boldsymbol{q}}}
\newcommand{\hs}{\ensuremath{\boldsymbol{h}}}

\newcommand{\path}{\ensuremath{\mathsf{path}}}
\newcommand{\pa}{\ensuremath{\mathsf{pa}}}
\newcommand{\ch}{\ensuremath{\mathsf{ch}}}
\newcommand{\val}{\ensuremath{\mathsf{val}}}
\newcommand{\supp}{\ensuremath{\mathsf{supp}}}

\newcommand{\mregister}{\ensuremath{\mathsf{m}}}
\newcommand{\tregister}{\ensuremath{\mathsf{t}}}
\newcommand{\register}{\ensuremath{\mathsf{r}}}
\newcommand{\scope}{\ensuremath{{\phi}}}

\newcommand{\abs}[1]{\left\lvert#1\right\rvert}

\newcommand{\proj}[2]{\left.{#1}\right\rvert_{#2}}

\newcommand{\eat}[1]{}

\newcommand{\LineIfElse}[3]{     
    \State{ \algorithmicif\ {#1}\ \algorithmicthen\ {#2}\ \algorithmicelse\ {#3}}
}

\begin{document}

%
\runningtitle{Solving Marginal MAP Exactly by Probabilistic Circuit Transformations}

%

\twocolumn[

\aistatstitle{Solving Marginal MAP Exactly by \\ Probabilistic Circuit Transformations}

\aistatsauthor{ YooJung Choi \And Tal Friedman \And  Guy Van den Broeck }

\aistatsaddress{ Computer Science Department\\UCLA\\yjchoi@cs.ucla.edu \And  Computer Science Department\\UCLA\\tal@cs.ucla.edu \And Computer Science Department\\UCLA\\guyvdb@cs.ucla.edu } ]

\begin{abstract}
Probabilistic circuits (PCs) are a class of tractable probabilistic models that allow efficient, often linear-time, inference of queries such as marginals and most probable explanations (MPE). However, marginal MAP, which is central to many decision-making problems, remains a hard query for PCs unless they satisfy highly restrictive structural constraints. In this paper, we develop a pruning algorithm that removes parts of the PC that are irrelevant to a marginal MAP query, shrinking the PC while maintaining the correct solution. This pruning technique is so effective that we are able to build a marginal MAP solver based solely on iteratively transforming the circuit---no search is required. We empirically demonstrate the efficacy of our approach on real-world datasets.
\end{abstract}

\section{INTRODUCTION}

Probabilistic circuits (PCs) refer to a family of tractable probabilistic models that are known to be able to closely capture the probability space in density estimation tasks~\citep{DangPGM20, LiuTPM21, PeharzICML20, rooshenas2014learning}, while allowing tractable probabilistic inference of many useful queries~\citep{LiUAI21,pmlr-v161-yu21a,VergariNeurIPS21}. Perhaps the most widely supported queries for tractable inference by different kinds of PCs are: marginal inference, which computes the probability of a partial assignment; and the most probable explanations (MPE),\footnote{MPE is sometimes referred to as MAP (maximum a posteriori hypothesis). To avoid confusion, in this paper we will use the terms MPE and marginal MAP.} which computes for a given partial assignment (or evidence) the most likely state of all the remaining variables.

However, many related inference tasks remain hard even on those PCs tractable for marginals and MPE~\citep{rahman2021novel,rouhani2018algorithms}.
In particular, marginal MAP (maximum a posteriori hypothesis) is a closely related problem that still appears to be hard for most probabilistic circuits, despite being used in many applications including image segmentation, planning, and diagnosis, among others \citep{lee2014applying, kiselev2014policy,7760303}. A marginal MAP (MMAP) problem, unlike MPE, computes the most likely state of a subset of variables, while marginalizing out the others. Although these queries appear closely related, a PC that can tractably solve both marginals and MPE queries does not necessarily solve the marginal MAP tractably.
In fact, exactly solving marginal MAP is known to be NP-hard, even for tractable PCs~\citep{deCampos2011new}. 
This remains to be the case when solving it approximately~\citep{conaty2017approximation,mei2018maximum}.

Most existing marginal MAP solvers on PCs, especially exact solvers, are based on variations of branch-and-bound search~\citep{mei2018maximum,HuangChaviraDarwiche06}, as has been the case for exact marginal MAP solvers for probabilistic graphical models~\citep{park2002solving,marinescu2014and}.
In this paper, we propose a novel approach to marginal MAP inference: probabilistic circuit transformations.

In particular, we show that large parts of the circuit may be irrelevant to the marginal MAP problem at hand, and thus can be pruned away without affecting the solution. This in a sense ``specializes'' the PC to a particular MMAP instance and makes it more amenable to solving.
We then develop an efficient algorithm to determine which parts of the circuit can be safely pruned, using a novel edge bound.
Lastly, we propose an exact MMAP solver that leverages this pruning algorithm and iteratively transforms the PC structure until the MMAP solution can be easily read from it.
We show empirically on real-world benchmark datasets that our method can solve more marginal MAP instances with faster run time than existing solvers.

\section{BACKGROUND}
\label{sec:background}

We use uppercase letters ($X$) to denote random variables and lowercase letters ($x$) for their assignments. Sets of variables are denoted by bold uppercase letters ($\X$) and their joint assignments by bold lowercase letters ($\x$). For a binary random variable $X$, we use logical negation $\neg X$ to denote $X=0$. Lastly, we write the set of all values for $\X$ as $\val(\X)$.

\subsection{Marginal MAP}

Suppose $p(\X)$ is a probability distribution over a set of variables $\X$ which is partitioned into three subsets $\Q$, $\E$, and $\Hs$, referred to as the query, evidence, and hidden variables, respectively. Given some evidence $\e \in \val(\E)$, the marginal MAP problem $\mathsf{MMAP}(\Q,\e)$ is defined as follows:
\begin{equation*}
    \argmax_{\q\in\val(\Q)} p(\q,\e) \, = \,  \argmax_{\q\in\val(\Q)} \sum_{\hs\in\val(\Hs)}  p(\q,\hs,\e).
\end{equation*}
Note that if $\Hs$ is empty, this corresponds to an MPE (most probable explanations) problem. 

\subsection{Probabilistic Circuits}

A large family of tractable probabilistic models---including arithmetic circuits~\citep{darwicheJACM-POLY}, and-or search spaces~\citep{marinescu2005and}, probabilistic sentential decision diagrams~\citep{KisaVCD14}, cutset networks~\citep{rahman2014cutset}, and sum-product networks~\citep{poon2011sum}---are collectively referred to as \textit{probabilistic circuits}~(PCs)~\citep{PCTuto20}.

A probabilistic circuit $\PC$ over variables $\X$ is a directed acyclic graph (DAG) structure with parameters that defines a (possibly unnormalized) probability distribution over $\X$ in a recursive manner. Specifically, the DAG structure consists of leaf, product, and sum nodes. A leaf node is associated with a univariate function, denoted $f_n$, such as the indicator function $[X=1]$. 
Every input edge $(n,c)$ to a sum unit $n$ is also associated with a parameter $\theta_{n,c} > 0$.
Let $\ch(n)$ denote the set of children, or inputs, of an inner node~$n$. A PC node then recursively defines a distribution as the following:
\begin{equation*}
n(\x)=
\begin{cases}
    f_n(\x) &\text{if $n$ is a leaf node} \\
    \prod_{c\in\ch(n)} c(\x) &\text{if $n$ is a product node} \\
    \sum_{c\in\ch(n)} \theta_{n,c} \cdot c(\x) &\text{if $n$ is a sum node}
\end{cases}
\end{equation*}
We write $\PC(\x)$ to refer to $n(\x)$ where $n$ is the root of the PC $\PC$.

A key strength of probabilistic circuits is that they support tractable inference, enabled by certain structural constraints. In particular, \textit{smooth} and \textit{decomposable} PCs allow efficient computation of marginal probabilities. 
\begin{defn}
    A PC $\PC$ is \textit{smooth} if for every sum node, its children depend on the same set of variables. A PC $\PC$ is \textit{decomposable} if for every product node, its children depend on disjoint sets of variables.
\end{defn}
For a smooth and decomposable PC over variables $\X$, computing the marginal probability of some partial assignment $\q \in \val(\Q), \Q \subseteq \X$ amounts to the following procedure. A leaf node $n$ is evaluated as 1 if it does not depend on a variable in $\Q$, and as $f_n(\q)$ otherwise. Then we simply evaluate the circuit, taking (weighted) sums and products accordingly. 
For instance, consider the smooth and decomposable PC $\PC$ in \cref{fig:ex-pc} and a partial assignment $\q = \{X_1=1,X_2=0\}$. Then to compute the marginal $\PC(\q)$, we first set the leaf nodes labeled $\neg X_1$ and $X_2$ as 0, and all others as 1. Evaluating the circuit bottom up, we get the marginal probability $\PC(\q) = 0.222$.

In addition, probabilistic circuits satisfying more restrictive structural constraints even support efficient inference of marginal MAP and related queries~\citep{oztok2016solving,ChoiIJCAI17}. These structural constraints can be generalized into the notion of \textit{$\Q$-determinism}~\citep{ProbCirc20}.
\begin{defn}
    Suppose $\PC$ is a PC over variables $\X$ and let $\Q\subseteq\X$ be a subset. A sum node in $\PC$ is $\Q$-deterministic if computing the marginal probability for any partial assignment $\q\in\val(\Q)$ makes at most one of its children evaluate to a nonzero output. A PC $\PC$ is $\Q$-deterministic if all sum nodes containing variables in $\Q$ are $\Q$-deterministic.
\end{defn}
Then, solving a marginal MAP problem $\mathsf{MMAP}(\Q,\e)$ of a $\Q$-deterministic PC simply amounts to evaluating the circuit bottom-up similar to computing a marginal, except that every sum node that contains a variable in $\Q$ takes the weighted maximum of its inputs, instead of the weighted sum.

As one may intuit from the complexity of marginal MAP, enforcing this structural constraint on an arbitrary PC is an intractable task, as we also later demonstrate empirically.
Furthermore, even if one somehow learns or constructs a PC that satisfies $\Q$-determinism, this would support tractable marginal MAP only for this specific $\Q$. This is clearly infeasible in applications where one wishes to answer different marginal MAP queries using the probabilistic~model.

In the following sections, we assume a PC that satisfies smoothness and decomposability. 
Moreover, for simplicity of exposition, we consider only the marginal MAP problems without any evidence. This is because a given evidence can be incorporated into the PC by setting the leaf nodes (just like for computing marginals), and then we can equivalently solve the marginal MAP problem with no evidence on the resulting PC.

\begin{figure*}[!t]
    \centering
    \begin{subfigure}[b]{0.33\textwidth}
        \centering
        \includegraphics[width=.88\columnwidth]{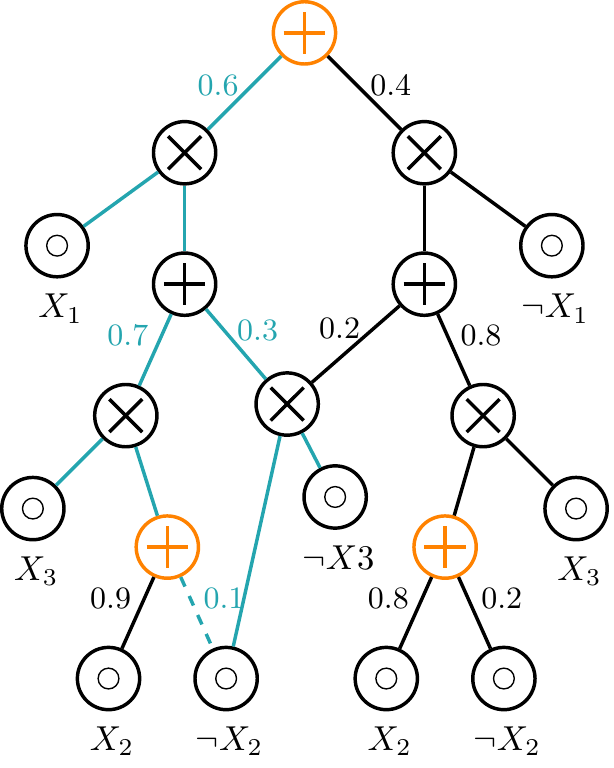}
        \caption{Example PC}\label{fig:ex-pc}
    \end{subfigure}
    \begin{subfigure}[b]{0.66\textwidth}
        \centering
        \includegraphics[width=.44\columnwidth]{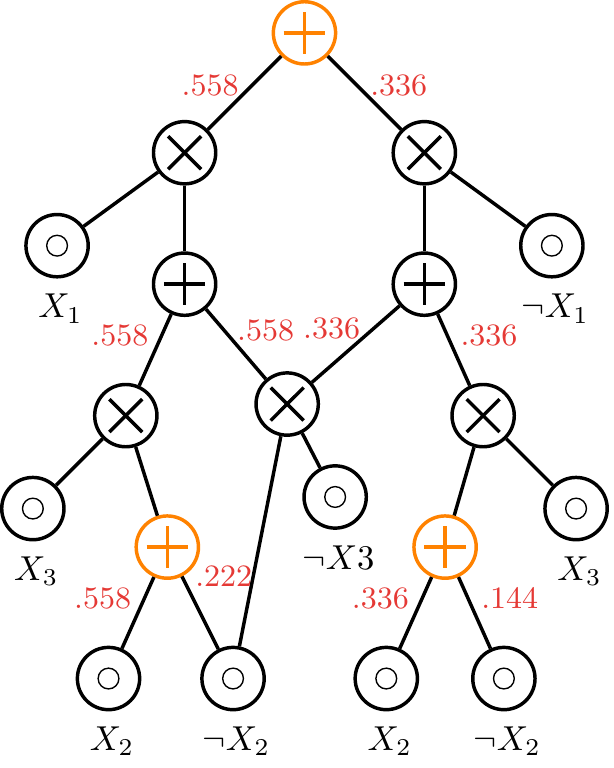}
        \hspace{1em}
        \includegraphics[width=.44\columnwidth]{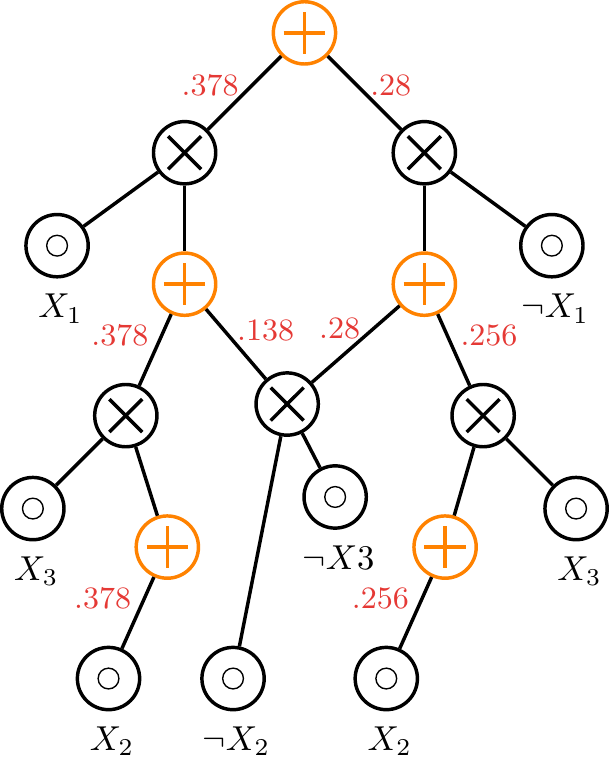}
    \caption{PC before and after pruning. Sum edges are labeled with edge bounds.}\label{fig:ex-bounds}
    \end{subfigure}
    \caption{A smooth and decomposable PC over variables $\{X_1,X_2,X_3\}$. Orange sum nodes are $\Q$-deterministic for $\Q=\{X_1,X_2\}$; blue edges form the sub-circuit for joint assignment $\q=\{X_1=1,X_2=0\}$.}
\end{figure*}

\section{CIRCUIT PRUNING FOR MARGINAL MAP}

We now describe the main contribution behind our proposed marginal MAP solver: pruning parts of a probabilistic circuit without affecting its MMAP solution. This is motivated by two key observations. 

\subsection{Motivation}

Consider the following two observations.

\textbf{(i):} Computing the marginal probability of any partial assignment $\q$ is equivalent to evaluating a sub-circuit in which every $\Q$-deterministic sum node has one input. In other words, the sub-circuit for $\q$ includes the parts of the PC that are used or ``activated'' when computing the marginal of $\q$. Let us call this the $\q$-subcircuit and denote it by $\PC^\prime_{\q}$.
We illustrate this with the example PC 
in \cref{fig:ex-pc}. 
Suppose $\Q = \{X_1,X_2\}$ and we wish to compute the marginal probability of $\q=\{X_1=1,X_2=0\}$. Recall from \cref{sec:background} that this corresponds to setting the input units for $\neg X_1$ and $X_2$ to 0 and all others to 1, then evaluating the circuit in a bottom-up fashion. We can quickly check that the output is $0.6 \cdot (0.7\cdot0.1 + 0.3) = 0.222$, which is equivalent to simply evaluating the sub-circuit highlighted in blue with its input units set to 1. Moreover, observe that every $\Q$-deterministic sum node (highlighted in orange) that is included in this sub-circuit has exactly one input.

\textbf{(ii):} If we remove an edge that does not appear in the sub-circuit for any assignment $\q$, then the (unnormalized) probability of $\q$ is unchanged in the resulting PC. 
This directly follows from observation (i).
For example, removing any non-colored edge from the PC in \cref{fig:ex-pc} does not affect the marginal for $\q$, as defined previously, in the resulting circuit.
Moreover, if an edge in the sub-circuit for $\q$ is removed, then the probability of $\q$ decreases 
in the resulting PC. Again visiting \cref{fig:ex-pc}, removing the edge represented by the dashed line will drop the probability of $\q=\{X_1=1,X_2=0\}$ from $0.222$ to $0.6\cdot 0.3 = 0.18$.

We can apply observations (i) and (ii) to the marginal MAP state, denoted by $\q^\star$, to conclude that any edge that does not appear in the $\q^\star$-subcircuit (namely the ``solution sub-circuit'') can be pruned away while keeping the MMAP problem equivalent.
That is, removing an edge that is not in the solution sub-circuit will not affect the probability of $\q^\star$ but may decrease the probabilities of other assignments to $\Q$; hence, $\q^\star$ remains as the solution for marginal MAP problem in the pruned circuit. 
Solving a MMAP instance by solving the equivalent problem on a pruned circuit can have the following important benefits. 
First, the complexity of inference algorithms on PCs generally depends on the size of the circuit, and thus reducing the size by pruning edges is desirable. 
In addition, because pruning as described above keeps the marginal MAP probability while potentially decreasing other marginal probabilities, it effectively increases the gap between the solution and other states. This can not only lead to more iterations of pruning, further specializing the circuit to the MMAP problem, but also arguably make the problem easier to solve. 
For example, in the extreme case that all edges other than the solution sub-circuit are pruned, the resulting MMAP problem becomes trivial to solve. 

Given these benefits, we naturally raise the following question: \textit{can we efficiently determine which edges do not appear in the solution sub-circuit (i.e.\ $\q^\star$-subcircuit)?}
The challenge is to do this without knowing a priori the marginal MAP state $\q^\star$.
In the following section, we propose an algorithm that efficiently computes, for every edge, an upper bound on the output of any sub-circuit that includes the edge, which gives a positive answer to the previous question.

\subsection{Edge Bounds}
\label{sec:edge-bounds}

We will now define more formally our edge bounds and the algorithm to efficiently compute them.

\paragraph{Definition}
Abusing notation, let us denote by $\textsf{MMAP}(\proj{\Q}{(n,c)})$ the largest marginal probability obtainable by an assignment $\q$ whose $\q$-subcircuit includes the edge $(n,c)$.
Formally,
\begin{equation}
    \textsf{MMAP}(\proj{\Q}{(n,c)}) := \max_{\q: (n,c) \in \PC^\prime_{\q}} \PC(\q). \label{eq:edge-bound}
\end{equation}
Intuitively, this corresponds to a marginal MAP problem where the possible states have been reduced from $\val(\Q)$ to those that ``activate'' the edge $(n,c)$ when computing their marginal probability.
Moreover, suppose we define a hypothetical edge from the root to output, denoted $(\cdot,\text{root})$. Then by definition, the MMAP reduced to this edge, i.e.\ $\mathsf{MMAP}(\proj{\Q}{(\cdot,\text{root})})$ is simply the MMAP problem $\mathsf{MMAP}(\Q)$.
%

For each edge $(n,c)$, we wish to obtain an edge bound $\EB(n,c)$ that satisfies the following:
\begin{equation*}
    \textsf{MMAP}(\proj{\Q}{(n,c)}) \leq \EB(n,c).
\end{equation*}
Let us also introduce $\EB$ for each node $n$, which may be useful as intermediate quantities as will be apparent later.
\begin{equation*}
    \textsf{MMAP}(\proj{\Q}{n}) \leq \EB(n).
\end{equation*}
It is important to note that the edge bound $\EB(n,c)$ is not a bound on some output from the edge or either of the nodes connected by it.
Rather, it bounds from above the output of the PC at the root, using the edge to limit the state space.
Suppose we are given such edge bound; clearly, if we also have a lower bound on the marginal MAP probability, we can safely prune any edge whose $\EB$ is smaller than the given lower bound.

\paragraph{Computing the Edge Bound} 
To develop an edge bound with the properties described above, we first observe that every $\q$-subcircuit that includes an edge $(n,c)$ must also include the node $n$. Then, we can conclude that $\textsf{MMAP}(\proj{\Q}{(n,c)}) \leq \textsf{MMAP}(\proj{\Q}{n})$.
Suppose we have an upper bound on the MMAP reduced to node $n$. Such bound will also be at least as large as the MMAP reduced to edge $(n,c)$, and can be used as edge bound $\EB(n,c)$. However, $\EB(n,c)$ need not be as large as $\textsf{MMAP}(\proj{\Q}{n})$, so there may be some opportunity to tighten the bound going from $n$ to $(n,c)$.

As a base case of the top-down recursion, we need an upper-bound of MMAP at the root. For this, we use the algorithm by \citet{HuangChaviraDarwiche06}, shown in \cref{alg:out-bound}, which not only computes the upper-bound on marginal MAP at the root node but also bounds the output of every node, via a single feedforward pass on the PC. Formally, for every node $n \in \PC$ it computes an upper bound on:
\begin{equation}
    \max_{\q: n \in\PC^\prime_{\q}} n(\q) = \max_{\q \in \val(\Q)} n(\q), \label{eq:output-bound}
\end{equation}
and stores it in $\mregister_n$. Let us denote the upper-bound at the root by $\mregister_{\PC}$.

Intuitively, our proposed edge bound $\EB(n,c)$ aims to upper-bound the largest value returned by \cref{alg:out-bound} on a $\q$-subcircuit that includes the edge $(n,c)$. In other words, for each edge $(n,c)$, we bound from above the following:
\begin{equation*}
    \max_{\q:(n,c)\in\PC^\prime_{\q}} \mregister_{\PC^\prime_{\q}} \geq \mathsf{MMAP}(\proj{\Q}{(n,c)}).
\end{equation*}
$\EB(n)$ then similarly upper-bounds $\mathsf{MMAP}(\proj{\Q}{n})$.
It is worth pointing out that this bounds the output \textit{at the root} for states $\q$ that includes $n$ in their sub-circuits, whereas $\mregister_n$ by \cref{alg:out-bound} upper-bounds the output \textit{at each node}.

\begin{algorithm}[!t]
    \caption{$\textsc{Output-Bounds}(\PC, \Q)$} \label{alg:out-bound}
    \begin{algorithmic}[1]
    \Input{a smooth \& decomposable PC $\PC$ over variables $\X$ and a set of query variables $\Q\subset\X$}
    \Output{$\mregister_n$ storing output bounds for each node $n$}
    \State {$\mathsf{N}\leftarrow\textsc{FeedforwardOrder}(\PC)$}
    \For {\textbf{each} $n\in\mathsf{N}$}
    \If{$n$ \text{is an input unit}}
        \State{$\mregister_{n}\leftarrow\PC_{n}^{\max}(\x_{\scope(n)})$}
    \ElsIf{$n$ \text{is a product unit}}
        \State{$\mregister_{n}\leftarrow\prod_{c\in\ch(n)}\mregister_{c}$}
    \ElsIf {$n$ \text{is $\Q$-deterministic}} \label{line:max}
        \State{$\mregister_{n}\leftarrow\max_{c\in\ch(n)}\theta_{n,c}\mregister_{c}$}
    \Else
        \State{$\mregister_{n}\leftarrow\sum_{c\in\ch(n)}\theta_{n,c}\mregister_{c}$}
    \EndIf
    \EndFor
    \end{algorithmic}
\end{algorithm}

Let us now describe the recursive steps. First, suppose we want to compute $\EB(n)$ where $\EB(p,n)$ for every parent $p$ of $n$ (i.e.\ $n \in\ch(p)$)  has been computed by the recursion. In order to make sure that $\EB(n)$ upper-bounds the marginal MAP reduced to $n$, we observe that if $\q$ is the solution to $\textsf{MMAP}(\proj{\Q}{n})$ then the $\q$-subcircuit must also include one of the edges $(p,n)$. 
Thus, $\EB(n) = \max_p(\EB(p,n))$ is a valid edge bound:
\begin{align*}
    \max_{p:n\in\ch(p)} \EB(p,n)
    &\geq \max_{p:n\in\ch(p)} \textsf{MMAP}(\proj{\Q}{(p,n)}) \\
    &= \textsf{MMAP}(\proj{\Q}{n})
\end{align*}

Next, suppose we wish to compute $\EB(n,c)$ from a given $\EB(n)$. We consider the three possible cases of $n$ being a $\Q$-deterministic sum node, a non $\Q$-deterministic sum node, and a product node.
For the latter two cases, the edge bounds are simply propagated from the node.
This is because any sub-circuit that includes such node will also include both of its input edges, and thus their bounds will be the same.

Finally, we consider the edge bound $\EB(n,c)$ for an input edge to a $\Q$-deterministic sum node. To illustrate the intuition, we use the example PC in \cref{fig:ex-pc}. 
Suppose we want the edge bound between the root and its right input, denoted $\EB(\text{root},r)$. 
%
Running \cref{alg:out-bound}, we get the upper bound $\mregister_{\text{root}} = 0.558$ at the root and $\mregister_{l} = 0.93$ and $\mregister_{r} = 0.84$ for its left and right input, respectively.
Note that for every $\q$ 
that includes this edge in its sub-circuit,\footnote{This corresponds to $\{X_1=0,X_2=0\}$ and $\{X_1=0,X_2=1\}$.} the marginal $\PC(\q)$ must be $0.4 \cdot r(\q)$, leading to:
\begin{equation*}
    \max_{\q: (\text{root},r)\in\PC^\prime_{\q}} \PC(\q) = 0.4 \cdot \max_{\q: (\text{root},r)\in\PC^\prime_{\q}} r(\q) \leq 0.4 \cdot \mregister_r.
\end{equation*}
Thus, we can use $0.4 \cdot \mregister_r=0.336$ as the edge bound for $(\text{root}, r)$.
Similarly, we can derive the edge bound for $(\text{root},l)$ as $0.6 \cdot \mregister_l=0.558$.
This can be expressed as:
\begin{align}
    \EB(\text{root},c) = \EB(\text{root}) - \mregister_{\text{root}} + \theta_{\text{root},c} \mregister_c \label{eq:eb-root}
\end{align}
for any $c\in\ch(\text{root})$. Note that this holds trivially because $\EB(\text{root})=\mregister_{\text{root}}$ as the base case.
However, we can generalize this to derive the expression for edge bound from an inner $\Q$-deterministic node.

Let us again use \cref{fig:ex-pc} as an example; this time we consider the blue dashed edge, denoting it $(n,c)$.
Recall that $\EB(n,c)$ aims to upper-bound what \cref{alg:out-bound} would return at the root of a $\q$-subciruit that includes edge $(n,c)$. In such sub-circuit, $(n,c)$ would be the only input edge to node $n$, and thus the algorithm would propagate up $\theta_{n,c} \mregister_c=0.1$ instead of $\mregister_n$.
This hints at a similar expression as \cref{eq:eb-root} where we subtract the contribution of $\mregister_n$ and add $\theta_{n,c}\mregister_c$.
However, a key observation is that $\mregister$ bounds from \cref{alg:out-bound} concern the output of each node, whereas the edge bounds concern the output of the root node.
Thus, we need to consider how the contribution of $\mregister_n$ gets scaled when it is propagated up to the root node. In this instance, it would be multiplied by $0.7 \cdot 0.6$, which is the product of edge parameters that lie in the path from $n$ to the root.
In other words, we get the following expression:
\begin{align*}
    \EB(n,c) = \EB(n) + 0.7\cdot 0.6 (-\mregister_n + \theta_{n,c} \mregister_c)
\end{align*}
The pseudocode for this recursive algorithm is described in \cref{alg:bound}.

\begin{algorithm}[!t]
    \caption{$\textsc{Edge-Bounds}(\PC, \Q)$} \label{alg:bound}
    \begin{algorithmic}[1]
    \Input{a smooth \& decomposable PC $\PC$ over variables $\X$ and a set of query variables $\Q\subset\X$}
    \Output{$\register_{n,c}$ storing edge bounds for each edge $(n,c)$}
    \State {$\mregister \leftarrow \textsc{Output-Bounds}(\PC, \Q)$}
    \State {$\tregister_{\text{root}} \leftarrow 1$}
    \State {$\register_{\text{root}} \leftarrow \mregister_{\text{root}}$}
    \State {$\mathsf{N}\leftarrow\textsc{BackwardOrder}(\PC)$}
    \For {\textbf{each} $n\in\mathsf{N}$ s.t.\ $\tregister_n>0$, $c \in \ch(n)$}
        \If{$n$ \text{is a product unit}}
            \State{$\register_{n,c} \leftarrow \register_n$}
            \State{$\register_c \leftarrow \max(\register_c,\register_{n,c})$}
            \State{$\tregister_c \leftarrow \min(\tregister_c,\tregister_n)$}
        \ElsIf {$n$ \text{is a sum unit}}
            \If {$n$ \text{is $\Q$-deterministic}} 
                \State{$\register_{n,c} \leftarrow \register_n + \tregister_n(\theta_{n,c} \mregister_c - \mregister_n)$} \label{line:det-edge}
            \Else 
                \State{$\register_{n,c} \leftarrow \register_n$}
            \EndIf
            \State{$\register_c \leftarrow \max(\register_c,\register_{n,c})$} 
            \State{$\tregister_c \leftarrow \min(\tregister_c,\theta_{n,c} \tregister_n)$}
        \EndIf
    \EndFor    
    \end{algorithmic}
\end{algorithm}

\begin{prop}\label{prop:bound}
    Given a smooth and decomposable PC $\PC$ over variables $\X$ and a subset $\Q\subset\X$, \cref{alg:bound} computes an upper bound on \cref{eq:edge-bound} for every edge in $\PC$.
\end{prop}

\paragraph{Pruning example}
We refer to the Appendix for a formal proof of the above proposition, and instead conclude this section with an example round of pruning.
Suppose we wish to prune edges from the PC in \cref{fig:ex-pc}, for an MMAP problem with $\Q=\{X_1,X_2\}$. First, we compute the edge bounds as shown in the left circuit in \cref{fig:ex-bounds}. To perform pruning, we need a lower bound on the marginal MAP probability to compare against. The probability of any $\q\in\val(\Q)$ state suffices; suppose we use $\q=\{X_1=0,X_2=1\}$ with $p(\q)=0.256$. 
Then we can prune two edges, resulting in the circuit on the right in \cref{fig:ex-bounds}. More notably, all sum nodes in the resulting circuit become $\Q$-deterministic (highlighted in orange). In particular, as we will discuss more in the next section, this allows us to answer the marginal MAP query via a single feedforward pass. Running \cref{alg:out-bound} on this PC, the output at the root is $0.378$ which exactly corresponds to the marginal MAP solution $p(X_1=1,X_2=1)=0.378$.

Thus, pruning not only has the immediate effect of decreasing the circuit size, but also changes the PC and its distribution in such a way that can make it easier to solve the marginal MAP problem.

%

\section{ITERATIVE MARGINAL MAP SOLVER}
\label{sec:solver}

We are now ready to show how the pruning algorithm from the previous section can be leveraged to solve marginal MAP exactly. 

As discussed briefly in \cref{sec:background}, we can tractably answer a marginal MAP query for a $\Q$-deterministic PC. Thus, a naive solver may try to transform the input PC into a $\Q$-deterministic one to solve a marginal MAP instance.
For example, one could apply the \textit{split} operation~\citep{LiangUAI17,DangPGM20} on the root for each variable in $\Q$. Splitting on a variable $Q\in\Q$ effectively turns the root of the PC into a $Q$-deterministic sum node while maintaining the distribution represented by it; thus, splitting on every variable in $\Q$ would result in a $\Q$-deterministic circuit.
However, this would be highly intractable as each split operation could at most double the size of the PC.

Instead, we propose to prune the circuit as well as split on a query variable in each iteration. While the circuit could grow exponentially in the worst case, we show empirically in the next section that pruning plays a crucial role in indeed keeping the circuit size from growing too much. In fact, in many instances, it decreases the circuit size over the iterations.

A pseudocode of our approach is shown in \cref{alg:solver}.\footnote{Our marginal MAP solver is implemented in https://github.com/Juice-jl/ProbabilisticCircuits.jl.}
The solver maintains an upper and lower bound
on marginal MAP and updates it after every prune and split. 
The upper bound is computed using \cref{alg:out-bound} as discussed in \cref{sec:edge-bounds}. The marginal probability of any instantiation of $\Q$ can be used as a lower bound on the MMAP probability. In particular, we use the solution to a different MMAP instance whose query variables include $\Q$ and can be solved efficiently; more details can be found in the Appendix.
In each iteration, we first prune all edges whose edge bound, computed by \cref{alg:bound}, does not exceed the current lower bound. Then we split on a variable chosen according to some heuristic (discussed further in the next section). 
The solver is guaranteed to converge after at most $\abs{\Q}$ iterations, at which point the PC must be $\Q$-deterministic, allowing exact computation of MMAP.
Furthermore, each prune and split improves the bounds, and thus the solver may also terminate before splitting on all query variables.
That is, pruning can decrease the upper bound as we saw in \cref{fig:ex-bounds},
and a split operation also improves the bounds by adding a new $\Q$-deterministic node at the root.
Lastly, we again emphasize that our marginal MAP solver only assumes smoothness and decomposability; determinism is not required. For example, this implies that we can also exactly solve MPE for non-deterministic PCs.

\begin{algorithm}[!t]
    \caption{$\textsc{Iter-Solve}(\PC, \Q)$} \label{alg:solver}
    \begin{algorithmic}[1]
    \Input{a smooth \& decomposable PC $\PC$ over variables $\X$ and a set of query variables $\Q\subset\X$}
    \Output{$\mathsf{MMAP}(\Q)$}
    \State {$u \leftarrow \textsc{Output-Bounds}(\PC, \Q)$}
    \State {$l \leftarrow \textsc{Lower-Bound}(\PC, \Q)$}
    \State {$\mathbf{V} \leftarrow \Q$}
    \While {$u > l$}
        \State{$\register \leftarrow \textsc{Edge-Bounds}(\PC,\Q)$}
        \ForAll{$(n,c)\in\PC \text{ s.t. } \register_{n,c} \leq l$}
            \State{$\PC \leftarrow \textsc{Prune-Edge}(\PC,(n,c))$}
        \EndFor
        \State{$X \leftarrow \textsc{Pick-Var}(\mathbf{V});\, \mathbf{V} \leftarrow \mathbf{V}\setminus\{X\}$}
        \State{$\PC \leftarrow \textsc{Split}(\PC,X)$}
        \State {$u \leftarrow \min(u, \textsc{Output-Bounds}(\PC, \Q))$}
        \State {$l \leftarrow \max(l, \textsc{Lower-Bound}(\PC, \Q))$}    \EndWhile
    \State{\Return $u$}
    \end{algorithmic}
\end{algorithm}

\begin{table*}[t]
    \caption{Average run time in seconds (with 1-hour time limit for each instance) and the number of instances solved for different proportions of (query, evidence, hidden) variables.
    }\label{tab:exp}
    \centering
    \scalebox{1.00}{
    \begin{tabular}{l | r@{\hspace{0.3em}}r r@{\hspace{0.3em}}r r@{\hspace{0.3em}}r | r@{\hspace{0.3em}}r r@{\hspace{0.3em}}r r@{\hspace{0.3em}}r}
        \toprule
        & \multicolumn{6}{c}{\textbf{(30\%, 30\%, 40\%)}} & \multicolumn{6}{c}{\textbf{(50\%, 20\%, 30\%)}} \\
        \textbf{Dataset} & \multicolumn{2}{c}{\textbf{\textsf{MaxSPN}}} & \multicolumn{2}{c}{\textbf{\textsf{(Pruned)}}} & \multicolumn{2}{c}{\textbf{\textsf{(UB)}}}  & \multicolumn{2}{c}{\textbf{\textsf{MaxSPN}}} & \multicolumn{2}{c}{\textbf{\textsf{(Pruned)}}} & \multicolumn{2}{c}{\textbf{\textsf{(UB)}}} \\
        \midrule
        NLTCS & \textbf{0.004} & (10) & 0.35 & (10) & 0.54 & (10) & \textbf{0.01} & (10) & 0.39 & (10) & 0.63 & (10) \\
        MSNBC & \textbf{0.01} & (10) & 0.29 & (10) & 0.50 & (10) & \textbf{0.03} & (10) & 0.43 & (10) & 0.73 & (10) \\
        KDD & \textbf{0.02} & (10) & 0.42 & (10) & 0.64 & (10) & \textbf{0.04} & (10) & 0.49 & (10) & 0.68 & (10) \\
        Plants & \textbf{0.27} & (10) & 0.99 & (10) & 1.36 & (10) & 2.95 & (10) & \textbf{2.61} & (10) & 2.72 & (10) \\
        Audio & 188.59 & (10) & 16.57 & (10) & \textbf{2.87} & (10) & 2041.33 & (6) & 15.61 & (10) & \textbf{13.70} & \textbf{(10)} \\
        Jester & 265.50 & (10) & 16.16 & (10) & \textbf{6.17} & (10) & 2913.04 & (2) & 44.16 & (10) & \textbf{14.74} & \textbf{(10)} \\
        Netflix & 344.71 & (10) & 22.23 & (10) & \textbf{5.61} & (10) & -- & (0) & 936.83 & (10) & \textbf{47.18} & \textbf{(10)} \\
        Accidents & \textbf{0.54} & (10) & 2.00 & (10) & 2.00 & (10) & 109.56 & (10) & 19.81 & (10) & \textbf{15.86} & (10) \\
        Retail & \textbf{0.03} & (10) & 0.47 & (10) & 0.61 & (10) & \textbf{0.06} & (10) & 0.67 & (10) & 0.81 & (10) \\
        Pumsb-star & 273.70 & (10) & 106.04 & (10) & \textbf{6.04} & (10) & 2208.27 & (7) & 54.32 & (10) & \textbf{20.88} & \textbf{(10)} \\
        DNA & 2809.44 & (4) & 65.27 & (10) & \textbf{9.16} & \textbf{(10)} & -- & (0) & 2634.41 & (3) & \textbf{505.75} & \textbf{(9)} \\
        Kosarek & 1.60 & (10) & \textbf{0.81} & (10) & 0.98 & (10) & 48.74 & (10) & \textbf{2.65} & (10) & 3.41 & (10) \\
        MSWeb & 25.70 & (10) & 3.63 & (10) & \textbf{0.96} & (10) & 1543.49 & (10) & 48.89 & (10) & \textbf{1.28} & (10) \\
        Book & -- & (0) & 56.47 & (10) & \textbf{7.25} & \textbf{(10)} & -- & (0) & 907.51 & (9) & \textbf{46.50} & \textbf{(10)} \\
        EachMovie & -- & (0) & 2563.02 & (3) & \textbf{93.66} & \textbf{(10)}  & -- & (0) & 3293.78 & (1) & \textbf{1216.89} & \textbf{(8)} \\
        WebKB & -- & (0) & 3378.03 & (2) & \textbf{102.37} & \textbf{(10)} & -- & (0) & -- & (0) & \textbf{575.68} & \textbf{(10}) \\
        Reuters-52 & -- & (0) & 1238.10 & (7) & \textbf{22.91} & \textbf{(10)} & -- & (0) & 3107.57 & (3) & \textbf{120.58} & \textbf{(10)} \\
        20 NewsGrp.& -- & (0) & 2882.95 & (3) & \textbf{88.13} & \textbf{(10)} & -- & (0) & -- & (0) & \textbf{504.52} & \textbf{(9)} \\
        BBC & -- & (0) & -- & (0) & \textbf{766.93} & \textbf{(9)} & -- & (0) & -- & (0) & \textbf{2757.18} & \textbf{(3)} \\
        Ad & -- & (0) & -- & (0) & \textbf{344.81} & \textbf{(10)} & -- & (0) & -- & (0) & \textbf{1254.37} & \textbf{(8)} \\
        \midrule
        Total Solved & & 124 & & 155 & & \textbf{199} & & 105 & & 146 & & \textbf{187} \\
        \bottomrule
    \end{tabular}}
\end{table*}

\section{EXPERIMENTS}\label{sec:exp}

We evaluated the iterative solver on probabilistic circuits learned from twenty widely-used benchmark datasets. The number of variables ranges from 16 to 1,556, and the size of PCs, learned using Strudel~\citep{DangPGM20}, ranges from 3,177 to 745,815.
We generated marginal MAP instances with two different proportions of query, evidence, and hidden variables---30\%, 30\%, 40\% and 50\%, 20\%, 30\%, respectively---randomly dividing the variables and generating evidence while ensuring its probability is nonzero. We generated 10 instances for each dataset and each proportion.

On each instance, we run our iterative solver with two different variable split heuristics. \textsf{(Pruned)} selects variables based on the number of pruned edges associated with the variable; \textsf{(UB)} selects variables by the expected change in upper bound after splitting on a variable, which can be computed efficiently via a single pass on the circuit. We refer to the Appendix for a more detailed description of split heuristics and algorithms to compute them.
For comparison, we also solved the marginal MAP problems using \textsf{MaxSPN}\footnote{Specifically, we use the forward checking technique with ordering and stage, which was shown to be the best performing among the exact solvers by \citet{mei2018maximum}.} which is a search-based exact solver for (marginal) MAP on sum-product networks~\citep{mei2018maximum}.
All experiments were ran on a Intel(R) Xeon(R) Gold 5220 CPU @ 2.20GHz.

\begin{figure}[!t]
    \centering
    \includegraphics[width=.75\columnwidth]{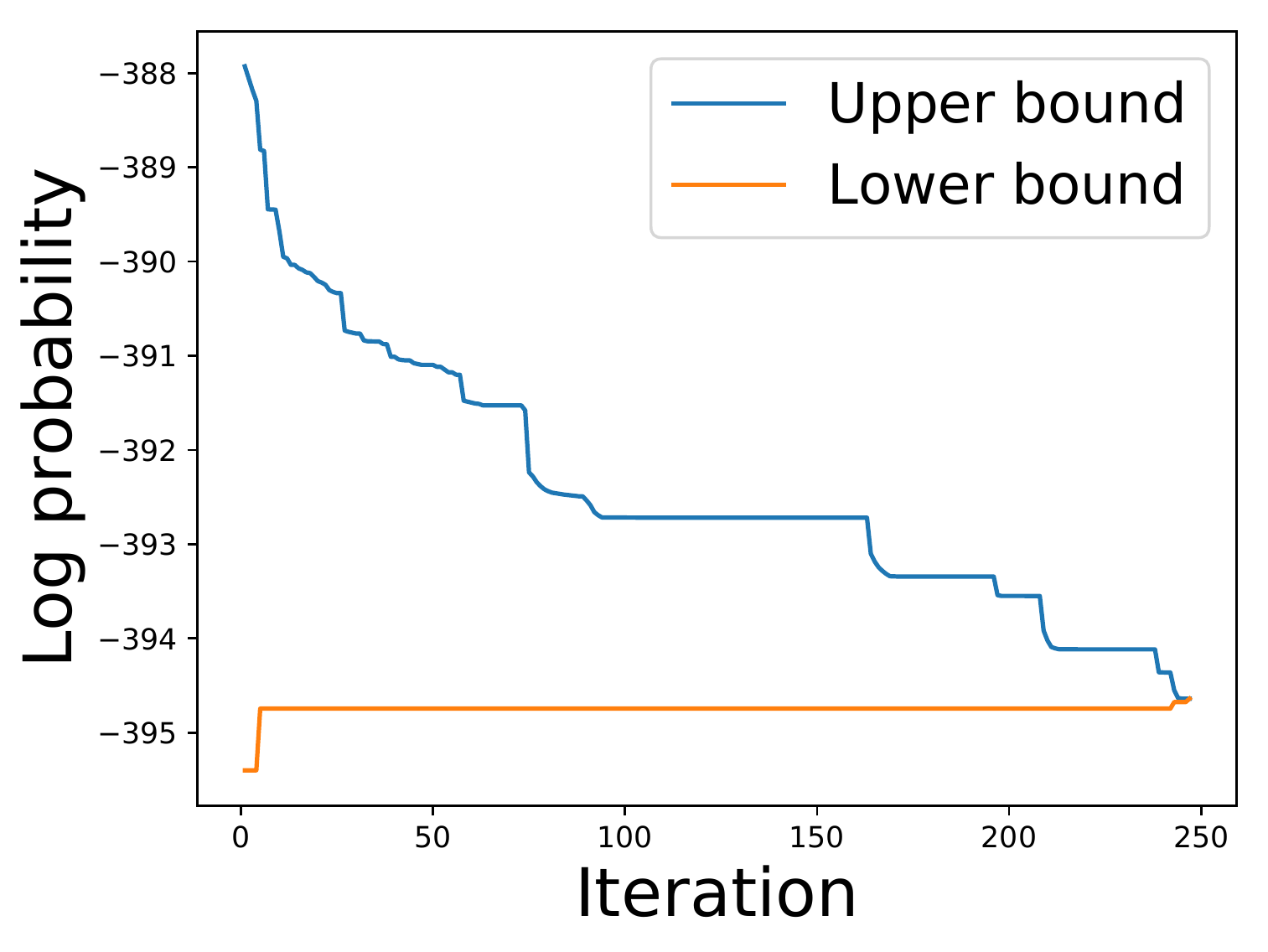}\\
    \includegraphics[width=.75\columnwidth]{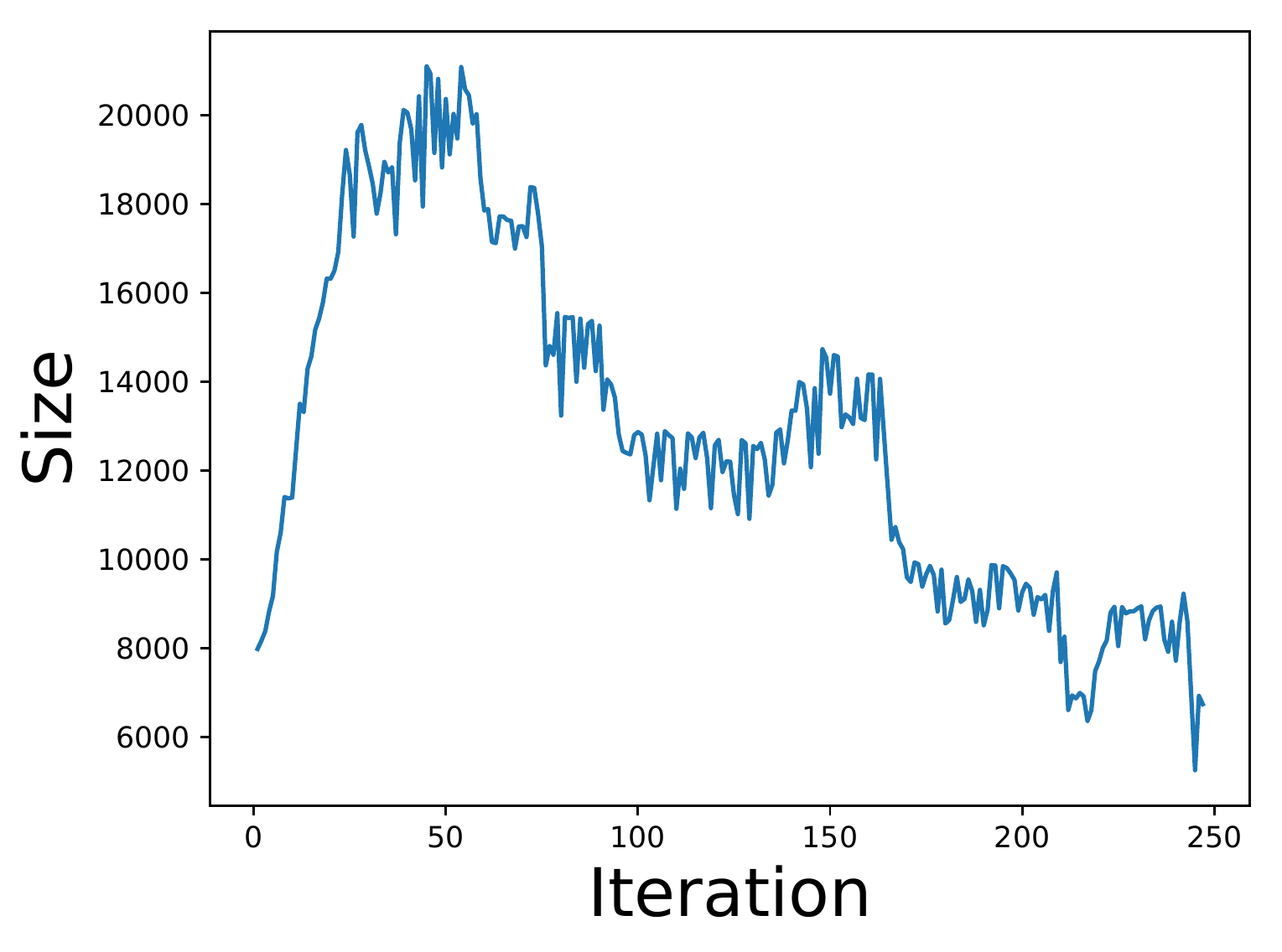}
    \caption{Upper and lower bounds (top) and circuit size (bottom) in each iteration of the solver on an example instance on EachMovie dataset.}
    \label{fig:ex-run}
\end{figure}

\cref{tab:exp} summarizes the results. 
First, we compare the two heuristics. \textsf{(Pruned)} is comparable or faster than \textsf{(UB)} on relatively easy datasets, but is significantly slower on most of the datasets. Moreover, \textsf{(Pruned)} failed to solve any instance on BBC and Ad datasets, whereas \textsf{(UB)} was able to solve at least one instance in all datasets. In fact, it was able to solve all 20 instances (10 for each proportion) on 15 out of the 20 datasets. This clearly demonstrates the importance of variable split heuristics and the benefit of explicitly choosing splits that lead to better bounds.

Next, we compare our iterative solver to the search-based approach of \textsf{MaxSPN}. We observe that \textsf{MaxSPN} is faster than our algorithm on easy instances (sub-1 second average run time). This is likely because there is a minimum overhead of performing circuit transformations.
On the other hand, our iterative approach clearly outperforms \textsf{MaxSPN} on all other datasets, both in terms of average run time and the number of instances solved.

Lastly, we examine more closely an example run of our solver to empirically demonstrate the benefits of pruning a PC for a specific marginal MAP problem; see \cref{fig:ex-run}. As we expected, iterative prune and split improve the upper and lower bounds until they converge. 
The next plot on circuit size clearly illustrates the importance of pruning the circuit. Even though split operations can increase the circuit size, we are very effective at pruning away irrelevant parts of the circuit for MMAP that the circuit size actually decreases over time. Indeed, the size at the point of convergence is smaller than the initial size. 
Judging by the rate of increase in the early iterations, it is not hard to imagine that without pruning, the circuit would quickly grow too large to run any inference.

\section{CONCLUSION}
We have introduced a novel approach to marginal MAP inference on probabilistic circuits. It is fundamentally distinct from existing solvers, which are based on a branch and bound search 
\citep{maua2020two, mei2018maximum, HuangChaviraDarwiche06}
using the tractable circuit to prune the search. Instead, we showed that the circuit can be pruned by keeping edges that are relevant to the marginal MAP state. Furthermore, our edge bounds algorithm can effectively find such edges to prune. What remains to solve marginal MAP is to perform simple splits on the circuit, tightening the bounds, and providing more opportunity to prune edges, until a marginal MAP solution is found. Our experiments empirically show that this novel approach to marginal MAP outperforms the search-based approach on a large number of real-world learned probabilistic circuits.

\subsubsection*{Acknowledgements}
This work is partially supported by a DARPA PTG grant, NSF grants \#IIS-1943641, \#IIS-1956441, \#CCF-1837129, Samsung, CISCO, and a Sloan Fellowship.

\bibliographystyle{plainnat}
\bibliography{references}




\clearpage
\appendix
\thispagestyle{empty}

\onecolumn \makesupplementtitle

\section{In-depth Look at Circuit Pruning}

\subsection{$\q$-subcircuit}

We first formally define the notion of $\q$-subcircuit used throughout the paper. This is expressed through the notion of contexts.
\begin{defn}[Context]\label{def:ctx}
    Let $\PC$ be a PC over variables $\X$ and $n$ be one of its nodes. The \emph{context} $\gamma_n$ of node $n$ denotes all joint assignments that return a nonzero value for all nodes in a path between the root of $\PC$ and $n$. 
    \begin{equation*}
        \gamma_n := \bigcup_{p\in\pa(n)} \gamma_p \cap \supp(n)
    \end{equation*}
    where $\pa(n)$ refers to the parent nodes of $n$ and $\supp(n) := \{\x : \PC_n(\x)>0\}$ is the support of node $n$. The context $\gamma_{(n,c)}$ of an edge $(n,c)$ is defined as $\gamma_{(n,c)} := \gamma_n \cap \gamma_c$.
\end{defn}
Then for any $\q$, an edge $(n,c)$ is said to be in the $\q$-subcircuit if $\q \in \proj{\gamma_{(n,c)}}{\Q}$; i.e., the context of $(n,c)$ reduced to variables in $Q$ contains the assignment $\q$.

\subsection{Proof of \cref{prop:bound}}

\reprop{prop:bound}{}{
    Given a smooth and decomposable PC $\PC$ over variables $\X$ and a subset $\Q\subset\X$, \cref{alg:bound} computes an upper bound on \cref{eq:edge-bound} for every edge in $\PC$.
}


To prove above proposition, let us define some auxiliary circuit structures. 
First, running \cref{alg:out-bound} to compute $\mregister$ can be interpreted as a feedforward evaluation on a circuit obtained from $\PC$ by replacing every $\Q$-deterministic sum node $n$ with a node that simply returns the output of child node $c=\argmax_{c\in\ch(n)} \theta_{n,c}\mregister_c$ (i.e.\ they are ``fixed'' to select the same branch as Line~\ref{line:max} in Algorithm~\ref{alg:out-bound}). Suppose we unroll such circuit into a tree structure: i.e.\ create copies of any node with multiple parents and recurse down. 
We denote this circuit by $\Mcircuit$. Then we have $\mregister_{\text{root}} = \Mcircuit(\emptyset)$, where $\Mcircuit(\emptyset)$ represents the circuit evaluation for marginal with no evidence.
Moreover, for any node $n^\prime$ in $\Mcircuit$ that corresponds to node $n$ in $\PC$, written as $n^\prime\in\text{copy}(n)$, we have $\mregister_n=\Mcircuit_{n^\prime}(\emptyset)$. 


In addition, for every node $n^\prime$ in $\Mcircuit$, we define a circuit denoted $\Mcircuit^{(n^\prime)}$ obtained from $\Mcircuit$ by ``fixing'' the $\Q$-deterministic nodes that appear in the path from root to $n^\prime$ such that they select the branch that reaches $n^\prime$. In other words, let $\Q^\prime$ be $\Q\setminus\scope(n^\prime)$ and $\q^\prime=\proj{\gamma_{n^\prime}}{\Q^\prime}$. Note that because $\Mcircuit$ is a tree structure, every assignment in the context of $n^\prime$ has the same value for variables in $\Q^\prime$; this is given by the $\Q$-deterministic sum nodes in the path from $n^\prime$ to the root which is unique.
Then $\Mcircuit^{(n^\prime)}$ is identical to $\Mcircuit$, except for the $\Q$-deterministic nodes that are ancestors of $n^\prime$, which output the child node whose context agrees with $\q^\prime$.

\begin{lem}\label{lem:mcircuit-n}
    Let $\PC$ be a PC over variables $\X$ and $\Mcircuit$ be its tree-unrolled max-sum circuit (as described above) for a set of query variables $\Q$. For any $\Mcircuit^{(n^\prime)}$ constructed from $\Mcircuit$ as above, the following statements hold:
    \begin{enumerate}
        \item $\Mcircuit^{(n^\prime)}(\emptyset) = \sum_{\y\in\val(\X\setminus\Q)} \Mcircuit^{(n^\prime)}(\y)$.
        \item For any $\q\in\left.{\gamma_{n^\prime}}\right\rvert_{\Q}$, $\Mcircuit^{(n^\prime)}(\emptyset) \geq \PC(\q)$.
    \end{enumerate}
\end{lem}
Note that above statements also apply to $\Mcircuit=\Mcircuit^{(\text{root})}$. 
%
%
We now provide a proof of \cref{prop:bound} using above lemma, which we will prove at the end of this section.
\begin{proof}
We will show that for every node $n$ in $\PC$, \cref{alg:bound} returns
\begin{align}
    \register_n \geq \max_{n^\prime\in\text{copy}(n)} \Mcircuit^{(n^\prime)}(\emptyset), \label{eq:node-register}
\end{align}
and for every edge $(n,c)$ it returns
\begin{align}
    \register_{n,c} \geq \max_{(n^\prime,c^\prime)\in\text{copy}((n,c))} \Mcircuit^{(c^\prime)}(\emptyset). \label{eq:edge-register}
\end{align}
Note that Equation~\ref{eq:edge-register} implies that $\register_{n,c}$ upper-bounds the quantity $\textsf{MMAP}(\proj{\Q}{(n,c)})$ given by Equation~\ref{eq:edge-bound}:
\begin{align*}
    \max_{(n^\prime,c^\prime)\in\text{copy}((n,c))} \Mcircuit^{(c^\prime)}(\emptyset)
    &\geq \max_{(n^\prime,c^\prime)\in\text{copy}((n,c))} \max_{\q\in\left.{\gamma_{c^\prime}}\right\rvert_{\Q}} \PC(\q) 
    = \max_{(n^\prime,c^\prime)\in\text{copy}((n,c))} \max_{\q\in\left.{(\gamma_{n^\prime}\cap\gamma_{c^\prime})}\right\rvert_{\Q}} \PC(\q) \\
    &= \max_{\q\in\bigcup_{(n^\prime,c^\prime)\in\text{copy}((n,c))}\left.{(\gamma_{n^\prime}\cap\gamma_{c^\prime})}\right\rvert_{\Q}} \PC(\q) 
    = \max_{\q\in\left.{(\gamma_{n}\cap\gamma_{c})}\right\rvert_{\Q}} \PC(\q)
    = \max_{\q\in\left.{\gamma_{(n,c)}}\right\rvert_{\Q}} \PC(\q) \\
    &= \max_{\q: (n,c) \in \PC^\prime_{\q}} \PC(\q) = \textsf{MMAP}(\proj{\Q}{(n,c)})
\end{align*}

We will now prove that Equations~\ref{eq:node-register} and \ref{eq:edge-register} hold by induction. For the base case, $\register_{\text{root}}$ is set as $\mregister_{\text{root}}$, which is exactly $\Mcircuit(\emptyset)=\Mcircuit^{(\text{root})}(\emptyset)$.

Next, assume Equation~\ref{eq:node-register} holds for a node n in $\PC$, and we want to show that Equation~\ref{eq:edge-register} holds for any of its input edges $(n,c)$. If $n$ is a product unit or a sum unit that is not $\Q$-deterministic, for any edge $(n,c)$ and its copy $(n^\prime,c^\prime)$ the circuits $\Mcircuit^{(n^\prime)}$ and $\Mcircuit^{(c^\prime)}$ are identical by definition. Then Equation~\ref{eq:edge-register} holds as follows:
\begin{align*}
    \max_{(n^\prime,c^\prime)\in\text{copy}((n,c))} \Mcircuit^{(c^\prime)}(\emptyset)
    &= \max_{(n^\prime,c^\prime)\in\text{copy}((n,c))} \Mcircuit^{(n^\prime)}(\emptyset)
    = \max_{n^\prime\in\text{copy}(n)} \Mcircuit^{(n^\prime)}(\emptyset) 
    \leq \register_n 
    = \register_{n,c}.
\end{align*}
If $n$ is a $\Q$-deterministic sum node, the circuits $\Mcircuit^{(n^\prime)}$ and $\Mcircuit^{(c^\prime)}$ can differ only by whether node $n^\prime$ is fixed to take $c^\prime$.
Thus, for any $\y\not\in\left.{\gamma_{n^\prime}}\right\rvert_{\Y}$ where $\Y=\X\setminus\Q$, $\Mcircuit^{(n^\prime)}(\y)=\Mcircuit^{(c^\prime)}(\y)$. For $\y\in\left.{\gamma_{n^\prime}}\right\rvert_{\Y}$, we have
\begin{align*}
    &\Mcircuit^{(n^\prime)}(\y) - \Mcircuit^{(c^\prime)}(\y) 
    = \Big(\prod_{\theta\in\path(n^\prime)} \theta \Big)\cdot \Mcircuit^{(n^\prime)}_{n^\prime}(\y) - \Big(\prod_{\theta\in\path(n^\prime)} \theta \Big) \cdot \theta_{n^\prime,c^\prime} \cdot \Mcircuit^{(c^\prime)}_{c^\prime}(\y)
\end{align*}
where $\path(n^\prime)$ denotes the set of all edge parameters that appear in the path from root to node $n^\prime$.
Note that $\Mcircuit^{(n^\prime)}_{n^\prime}$, i.e.\ the subcircuit of $\Mcircuit^{(n^\prime)}$ rooted at $n^\prime$, is identical to $\Mcircuit_{n^\prime}$ as the two max-sum circuits differ only in the ancestors of $n^\prime$. Similarly, $\Mcircuit^{(c^\prime)}_{c^\prime}$ is equal to $\Mcircuit_{c^\prime}$.
Then we can express the circuit evaluation of $\Mcircuit^{(c^\prime)}$ as
\begin{align*}
    \Mcircuit^{(c^\prime)}(\emptyset)
    &= \sum_{\y\in\val(\Y)} \Mcircuit^{(c^\prime)}(\y) 
    = \sum_{\y\not\in\left.{\gamma_{n^\prime}}\right\rvert_{\Y}} \Mcircuit^{(c^\prime)}(\y) + \sum_{\y\in\left.{\gamma_{n^\prime}}\right\rvert_{\Y}} \Mcircuit^{(c^\prime)}(\y) \\
    &= \sum_{\y\not\in\left.{\gamma_{n^\prime}}\right\rvert_{\Y}} \Mcircuit^{(n^\prime)}(\y) + \sum_{\y\in\left.{\gamma_{n^\prime}}\right\rvert_{\Y}} \Mcircuit^{(c^\prime)}(\y) 
    = \Mcircuit^{(n^\prime)}(\emptyset) - \sum_{\y\in\left.{\gamma_{n^\prime}}\right\rvert_{\Y}} \Mcircuit^{(n^\prime)}(\y) + \sum_{\y\in\left.{\gamma_{n^\prime}}\right\rvert_{\Y}} \Mcircuit^{(c^\prime)}(\y) \\
    &= \Mcircuit^{(n^\prime)}(\emptyset) + \left(\prod_{\theta\in\path(n^\prime)} \theta \right)\left( \theta_{n^\prime,c^\prime} \sum_{\y\in\left.{\gamma_{n^\prime}}\right\rvert_{\Y}} \Mcircuit_{c^\prime}(\y) -  \sum_{\y\in\left.{\gamma_{n^\prime}}\right\rvert_{\Y}} \Mcircuit_{n^\prime}(\y) \right)
    \\ &= \Mcircuit^{(n^\prime)}(\emptyset) + \Big(\prod_{\theta\in\path(n^\prime)} \theta \Big)\left( \theta_{n^\prime,c^\prime} \Mcircuit_{c^\prime}(\emptyset) - \Mcircuit_{n^\prime}(\emptyset) \right) 
    = \Mcircuit^{(n^\prime)}(\emptyset) + \Big(\prod_{\theta\in\path(n^\prime)} \theta \Big)\left( \theta_{n,c} \mregister_{c} - \mregister_{n} \right)
\end{align*}
Because $\mregister_{n} \geq \theta_{n,c} \mregister_{c}$, above equation among copies of $(n,c)$ can be bounded from above by:
\begin{align*}
    \max_{(n^\prime,c^\prime)\in\text{copy}((n,c))} \Mcircuit^{(c^\prime)}(\emptyset)
    &\leq
    \max_{n^\prime\in\text{copy}(n)} \Mcircuit^{(n^\prime)}(\emptyset) + \Big(\min_{n^\prime\in\text{copy}(n)} \prod_{\theta\in\path(n^\prime)} \theta  \Big)\left( \theta_{n,c} \mregister_{c} - \mregister_{n} \right).
\end{align*}
We will show that $\register_{n,c}=\register_n + \tregister_n \left( \theta_{n,c} \mregister_{c} - \mregister_{n} \right)$ (Line~\ref{line:det-edge} in Algorithm~\ref{alg:bound}) is at most the right-hand side quantity of above inequality, thereby satisfying Equation~\ref{eq:edge-register}.
First, we have 
$\register_n \geq \max_{(n^\prime)\in\text{copy}(n)} \Mcircuit^{(n^\prime)}(\emptyset)$ by the inductive hypothesis. 
Next, we want to show that $\tregister_n \leq \min_{n^\prime\in\text{copy}(n)} \prod_{\theta\in\path(n^\prime)} \theta $.
For a given node $c$, suppose this holds for $\tregister_n$ of every parent node $n\in\pa(c)$. 
Then we have
\begin{align*}
    \tregister_c = \min_{n\in\pa(c)} \theta_{n,c} \tregister_n
    \leq \min_{n\in\pa(c)} \theta_{n,c} \Big(\min_{n^\prime\in\text{copy}(n)} \prod_{\theta\in\path(n^\prime)} \theta\Big)
    = \min_{c^\prime\in\text{copy}(c)} \prod_{\theta\in\path(c^\prime)} \theta
\end{align*}
For simplicity, we say $\theta_{n,c}=1$ for a product node $n$.

Finally, assume that Equation~\ref{eq:edge-register} holds for edges $(p,n)$ where $p\in\pa(n)$, and we will show that Equation~\ref{eq:node-register} must hold then for node $n$. $\register_n$, which is set to $\max_{p\in\pa(n)} \register_{p,n}$ in Algorithm~\ref{alg:bound}, satisfies Equation~\ref{eq:node-register} as follows:
\begin{align*}
    \max_{p\in\pa(n)} \register_{p,n} 
    \geq \max_{p\in\pa(n)} \max_{(p^\prime,n^\prime)\in\text{copy}((p,n))} \Mcircuit^{(n^\prime)}(\emptyset)
    = \max_{n^\prime\in\text{copy}(n)} \Mcircuit^{(n^\prime)}(\emptyset).
\end{align*}
This concludes the proof of Proposition~\ref{prop:bound}.
\end{proof}


\begin{proof}[Proof of Lemma~\ref{lem:mcircuit-n}]
To show property (1) $\Mcircuit^{(n^\prime)}(\emptyset) = \sum_{\y\in\val(\X\setminus\Q)} \Mcircuit^{(n^\prime)}(\y)$, first observe that $\Mcircuit^{(n^\prime)}$ fixes every $\Q$-deterministic node to always return the value of one of its children and thus can be simplified by removing those nodes and directly connecting its parent to the appropriate child node. This results in a smooth and decomposable PC with the normal types of sum and product nodes. Then (1) simply holds by the fact that smooth and decomposable PCs allow marginal inference by feedforward evaluation.

Property (2) $\Mcircuit^{(n^\prime)}(\emptyset) \geq \PC(\q)$ holds for any $\q\in\proj{\gamma_{n^\prime}}{\Q}$ if and only if $\Mcircuit^{(n^\prime)}(\emptyset) \geq \PC_{\q}^\prime(\emptyset)$, as computing the marginal probability of $\q$ is equivalent to evaluating the $\q$-subcircuit. 
Note that for any $\q\in\proj{\gamma_{n^\prime}}{\Q}$, the ancestor nodes of $n^\prime$ in $\Mcircuit^{(n^\prime)}$ are equivalent to those in the $\q$-subcircuit. On the other hand, $\Mcircuit^{(n^\prime)}_{n^\prime}(\emptyset) =\Mcircuit_{n^\prime}(\emptyset)$ upper bounds $\max_{\q\in\proj{\gamma_{n^\prime}}{\Q}} n(\q)$ (recall \cref{eq:output-bound}), hence must be at least $n(\q)$.
Therefore, at the root nodes, $\Mcircuit^{(n^\prime)}$ must evaluate to at least $\PC_{\q}^\prime$.
\end{proof}


\subsection{MMAP Lower Bound}

\begin{algorithm}[!t]
    \caption{$\textsc{Lower-Bound}(\PC, \Q)$} \label{alg:lb}
    \begin{algorithmic}[1]
    \Input{a PC $\PC$ over variables $\X$ and a set of query variables $\Q\subset\X$}
    \Output{an assignment $\q\in\val(\Q)$}
    \State {$\mathsf{N}\leftarrow\textsc{FeedforwardOrder}(\PC)$}
    \For {\textbf{each} $n\in\mathsf{N}$}
    \If{$n$ \text{is an input unit}}
        $\mregister_{n}\leftarrow 1.0$
    \ElsIf{$n$ \text{is a product unit}}
        $\mregister_{n}\leftarrow\prod_{c\in\ch(n)}\mregister_{c}$
    \ElsIf {$n$ \text{or its descendant is $\Q$-deterministic}} 
        $\mregister_{n}\leftarrow\max_{c\in\ch(n)}\theta_{n,c}\mregister_{c}$
    \Else
        $\; \mregister_{n}\leftarrow\sum_{c\in\ch(n)}\theta_{n,c}\mregister_{c}$
    \EndIf
    \EndFor
    \State{\Return{$\textsc{Extract-state}(\PC,\Q,\mregister)$}}
    \item[]
    \Procedure{Extract-state}{n,\Q,\mregister}
        \If{$n$ \text{is an input unit}}
            \LineIfElse{$\text{Variable}(n) \in \Q$}{\Return{$\{\text{Literal}(n)\}$}}{\Return{$\{\}$}}
        \ElsIf{$n$ is a product unit}
            \State{\Return{$\bigcup_{c\in\ch(n)} \textsc{Extract-state}(c,\Q,\mregister)$}}
        \Else
            \State{\Return{$\textsc{Extract-state}(\argmax_{c\in\ch(n)} \theta_{n,c} \mregister_c,\Q,\mregister)$}}
        \EndIf
    \EndProcedure
    \end{algorithmic}
\end{algorithm}

As mentioned in \cref{sec:solver}, the solver maintains a lower bound on marginal MAP to be used for pruning. We now describe the algorithm to compute the lower bound used in our iterative solver. 
First, note that the probability of any assignment to query variables can be used as a lower bound for marginal MAP by definition.
A simple and common approach to approximate the marginal MAP state is to solve MPE instead and reduce the MPE state to the query variables.
We use a similar approach but with a key additional guarantee: after splitting on all query variables, it exactly solves the marginal MAP problem. A pseudocode of our method is shown in \cref{alg:lb}.
Note the similarity of its feedforward pass to \cref{alg:out-bound}: they both evaluate the circuit while replacing some sum nodes to take the weighted maximum. However, our algorithm not only replaces the $\Q$-deterministic sum nodes but all of their ancestors as well.
This is so that we can extract a state $\q$ by a backward pass, following the edges that were selected by the weighted maximum. 
Moreover, if the input PC $\PC$ is $\Q$-deterministic, this algorithm behaves the same as $\cref{alg:out-bound}$ and exactly solves the MMAP problem.

\section{Split Heuristics}

This section describes the two variable split heuristics that were evaluated in \cref{sec:exp}. 

Using the \textsf{(Pruned)} heuristic, at every iteration we split on the query variable that had the most number of associated edges pruned. In other words, for each query variable $Q\in\Q$ that is yet to be split on, we count how many edges of a $Q$-deterministic sum node have been pruned (this value can be cached to minimize redundant calculations) and choose the variable with the highest count. 
Intuitively, using this heuristic would tend to minimize a size blow-up by each split.

On the other hand, \textsf{(UB)} aims to maximize opportunities for pruning in the iteration following each split. 
To compute the heuristic, we first compute for each query variable $Q\in\Q$ the MMAP upper-bounds as described in \cref{alg:out-bound}, one setting $Q=0$ as evidence and the other $Q=1$. Because splitting the root on $Q$ would introduce a deterministic sum node whose children set $Q$ to 0 and 1, these bounds equal the edge bounds on the two input edges to the root after splitting. Let us denote these bounds $B_{Q=0}$ and $B_{Q=1}$ respectively, the lower bound in the current iteration as $lb$, and the candidate query variables by $\Q^\prime \subseteq \Q$ (i.e.\ query variables that have not been split on in the previous iterations). 
Then the \textsf{(UB)} heuristic selects a variable as follows:
\begin{align*}
    \begin{cases}
        \argmin_{Q: \min(B_{Q=0}, B_{Q=1}) < lb} \max(B_{Q=0}, B_{Q=1}) & \text{if $\exists\,Q\in\Q^\prime$ s.t. $\min(B_{Q=0}, B_{Q=1}) < lb$,} \\
        \argmin_{Q \in \Q^\prime} B_{Q=0} + B_{Q=1} & \text{otherwise}.
    \end{cases}
\end{align*}
In other words, if any variable would have a corresponding edge bound drop below the lower bound, we prioritize selecting from those variables as this guarantees a large part of the circuit is pruned in the next iteration. Then we choose the variable that would decrease the upper bound the most, which would, intuitively, result in more edges being pruned in the next iteration.
Note that computing this heuristic requires additional passes through the PC, but as we showed empirically in \cref{sec:exp}, it makes pruning much more effective and the resulting solver more efficient, despite the added time to compute the heuristic.


\end{document}